\documentclass[letterpaper,11pt]{article}
\usepackage[utf8]{inputenc}
\usepackage{times}  
\usepackage{helvet} 
\usepackage{courier}  
\usepackage[hyphens]{url}  
\usepackage{graphicx} 
\usepackage[ bottom=1.25in, top=1.25in, left=1.25in, right=1.25in]{geometry}
\usepackage{natbib} 
\usepackage{caption} 
\usepackage{authblk}
\usepackage[dvipsnames]{xcolor}
\usepackage[colorlinks = true,
            linkcolor = blue,
            urlcolor  = blue,
            citecolor = LimeGreen,
            anchorcolor = blue]{hyperref}
\usepackage{algorithm}
\usepackage{algorithmic}
\newcommand*\samethanks[1][\value{footnote}]{\footnotemark[#1]}
\title{Optimizing Group-Fair Plackett-Luce Ranking Models for Relevance and Ex-Post Fairness}
\date{}
\author[$1$]{Sruthi Gorantla}
\author[$2$]{Eshaan Bhansali\thanks{Work partially done while the author was a research intern at Microsoft Research, Bengaluru, India.}}
\author[$3$]{Amit Deshpande\thanks{Equal Contribution}}
\author[$1$]{Anand Louis\samethanks}

\affil[$1$]{\small{Indian Institute of Science, Bengaluru, India. \texttt{\{gorantlas, anandl\}@iisc.ac.in}}}
\affil[$2$]{\small{University of California, Berkeley, USA. \texttt{eshaanb@live.com}}}
\affil[$3$]{\small{Microsoft Research, Bengaluru, India. \texttt{amitdesh@microsoft.com}}}

\usepackage{amsthm, amssymb, amsmath, cleveref}
\newtheorem{theorem}{Theorem}[section]
\newtheorem{lemma}[theorem]{Lemma}
\newtheorem{definition}[theorem]{Definition}

\usepackage{enumitem}

\renewcommand{\le}{\leqslant}

\renewcommand{\ge}{\geqslant}

\newcommand{\paren}[1]{\left(#1 \right )}

\newcommand{\sparen}[1]{\left[#1 \right ]}

\newcommand{\set}[1]{\left\{#1\right\}}

\newcommand{\abs}[1]{\left\lvert#1\right\rvert}

\newcommand{\R}{\mathbb R}

\newcommand{\E}{\mathbb E}

\newcommand{\cI}{\mathcal{I}}

\newcommand{\cR}{\mathcal{R}}

\usepackage{caption}
\usepackage{subcaption}
\usepackage{rotating}
\usepackage{makecell, cellspace, multirow}

\newcommand{\sfair}{\textsf{S}^{\text{fair}}_k(\cI)}
\newcommand{\gfair}{\textsf{G}^{\text{fair}}_k(\ell)}
\newcommand{\rfair}{\cR^{\text{fair}}}
\newcommand{\rhat}{\widehat{\cR}}
\newcommand{\pfair}{\pi^{\text{fair}}}
\newcommand{\pipl}{\pi^{\text{PL}}}

\allowdisplaybreaks
\begin{document}

\maketitle
\begin{abstract}
In learning-to-rank (LTR), optimizing only the relevance (or the expected ranking utility) can cause representational harm to certain categories of items. Moreover, if there is implicit bias in the relevance scores, LTR models may fail to optimize for true relevance.
Previous works have proposed efficient algorithms to train stochastic ranking models that achieve fairness of exposure to the groups ex-ante (or, in expectation), which may not guarantee representation fairness to the groups ex-post, that is, after realizing a ranking from the stochastic ranking model.
Typically, ex-post fairness is achieved by post-processing, but previous work does not train stochastic ranking models that are aware of this post-processing.

In this paper, we propose a novel objective that maximizes expected relevance only over those rankings that satisfy given representation constraints to ensure ex-post fairness.
Building upon recent work on an efficient sampler for ex-post group-fair rankings, we propose a group-fair Plackett-Luce model and show that it can be efficiently optimized for our objective in the LTR framework. 

Experiments on three real-world datasets show that our group-fair algorithm guarantees fairness alongside usually having better relevance compared to the LTR baselines. In addition, our algorithm also achieves better relevance than post-processing baselines, which also ensures ex-post fairness. Further, when implicit bias is injected into the training data, our algorithm typically outperforms existing LTR baselines in relevance.
\end{abstract}

\section{Introduction}
% \paragraph{Fair LTR.} 
Rankings of people, places, news, and products have critical real-world applications that influence our worldview.
Ranking systems that optimize for relevance alone can amplify societal biases learned from their training data and reinforce certain stereotypes \cite{castillo2019fairness,zehlike2022fairness-part1,zehlike2023fairness-part2}. 
Thus, the field of fairness in learning-to-rank (LTR) has emerged as a response to these concerns, aiming to develop methodologies that ensure equitable and unbiased ranking outcomes.   

Stochastic ranking models have gained popularity in LTR \cite{cao2007learning,xia2008listwise,oosterhuis2018differentiable}, primarily due to off-the-shelf gradient-based methods that can be used to optimize these models efficiently.  
Further, they provide fairness guarantees that deterministic rankings for LTR cannot, e.g., ensuring that multiple items or groups have an equal (or some guaranteed minimum) probability of appearing at the top.

There are two types of fairness guarantees one could ask for in a stochastic ranking: \textit{ex-ante} and \textit{ex-post}.
Ex-ante fairness asks for satisfying fairness in expectation, i.e., before the stochastic ranking model realizes a ranking. 
In contrast, ex-post fairness, often referred to as outcome fairness, requires fairness of the actual ranking after one is generated by the stochastic ranking model\footnote{The choice between ex-post and ex-ante fairness depends on the context, available data, and additional ethical considerations.}.
To the best of our knowledge, recent works can guarantee ex-post fairness only by deterministic or randomized post-processing of the stochastic ranking model. But, in-processing for ex-post fairness has not been studied before this work.

We consider the well-known Plackett-Luce (PL) model as our stochastic ranking model.
PL model has been used in many fields, such as statistics \cite{plackett1975analysis,gormley2009grade}, psychology \cite{luce1959individual}, social choice theory \cite{soufiani2012random}, econometrics \cite{beggs1981assessing}, amongst others. Recent work has increased the popularity, scope, and efficiency of the PL model in LTR \cite{singh2019policy,diaz2020evaluating,oosterhuis2022learning}. 
It is also shown to be robust \cite{bruch2020stochastic} and effective for exploration in online LTR \cite{oosterhuis2018differentiable,oosterhuis2021unifying}. 

Faster and practical algorithms for novel and unbiased estimators of the gradient of the expected ranking utility--for the PL model--have been proposed recently
\cite{oosterhuis2021computationally,oosterhuis2022learning}.
These algorithms can efficiently optimize PL models for not just relevance (e.g., discounted cumulative gain) but also certain fairness notions that can be expressed as an expectation over the stochastic rankings (e.g., fair exposure).
Due to the inherent randomization in the PL model, ex-post fairness guarantees are more challenging to incorporate in the training process such that the resultant model can be optimized efficiently.

\subsection{Motivating Example for Ex-Post Fairness}
We start by demonstrating the importance of ex-post fairness in real-world ranking systems.
Consider a job recommendation platform such as \textit{LinkedIn Talent Search}\footnote{\url{https://business.linkedin.com/talent-solutions}}, where a stochastic ranking algorithm determines the order in which potential interview candidates from different demographic groups are recommended to recruiters.
Let us say there are candidates from two groups -- $G_1$, a majority group with high merit scores, and $G_2$, a minority group (usually underprivileged) whose merit scores are underestimated due to biases present in the training data used for LTR. 
These biases may originate from historical imbalances, social prejudices, or systemic inequalities in the data.

The stochastic ranking model must output the top-$10$ candidates every time a recruiter queries for a list of suitable candidates.
Consider a particular stochastic ranking that (1) chooses a group $G_1$ or $G_2$ with probability $0.5$ each, and (2) shows the top $10$ candidates from the group chosen in Step 1. 
This ensures \textit{equal representation} of both the groups ex-ante because there will be $5$ candidates in the top $10$ from each group, in expectation.
However, none of the rankings output by the stochastic ranking satisfies equal representation ex-post.
Such rankings may not be aligned with the ethical and diversity hiring policies of the recruiters (or companies).

% our contributions
\subsection{Our Contributions}
The main contribution of our work is a novel objective that maximizes relevance for a Group-Fair-PL model, where the relevance (or the expected ranking utility) is taken over only those rankings that satisfy given representation constraints for certain sensitive categories or groups of items. We show that a recent post-processing sampler for ex-post group-fair rankings \cite{gorantla2022sampling} combined with recent ideas to optimize the group-wise PL model \cite{oosterhuis2021computationally,oosterhuis2022learning} can be used to optimize this model efficiently. As a result, we get the best of both worlds: the efficiency of optimization in a fairness-aware in-processing objective and the ex-post fairness guarantees of post-processing methods.

Our experiments on three real-world datasets, HMDA, German Credit, and MovieLens, show that our model guarantees ex-post fairness and achieves higher relevance compared to the baselines. Implicit bias in training data can often negatively affect ranking models optimized for relevance \cite{celis2020interventions}. When implicit bias is injected into the training data as a stress test or audit for fair ranking algorithms, our algorithm outperforms existing baselines in fairness and relevance.

The rest of the paper is organized as follows: In Section \ref{sec:related}, we discuss closely related work in fair ranking to point out the significance and novelty of our results. Section \ref{sec:group-fair-plackett-luce} defines our novel relevance objective with ex-post fairness guarantees. In \Cref{sec:alg}, we show how to optimize the Group-Fair-PL model for our objective. 
Section \ref{sec:experiments} contains an experimental validation of our relevance and fairness guarantees.

\section{Related Work} \label{sec:related}
Stochastic ranking models have been widely studied in LTR, as they can be differentiable, and thus one can compute the gradient of a ranking utility to be optimized (e.g., discounted cumulative gain). In particular, the PL ranking model has been a popular model in recent work for optimizing relevance and fairness \cite{oosterhuis2021computationally,oosterhuis2022learning, singh2019policy}.
Recent work has proposed efficient and practical algorithms, namely, PL-Rank and its variants, for optimizing PL ranking models using estimates of the gradient \cite{oosterhuis2021computationally}. In addition to optimizing ranking utility, the PL-Rank algorithm also optimizes \textit{fairness of exposure} -- an ex-ante fairness metric \cite{singh2019policy}. \citet{yadav2021policy} also optimize a PL ranking model for both utility and fairness of exposure in the presence of position bias, where items that are ranked higher receive more positive relevance feedback. Similar to these works, we too study the PL ranking model for LTR. However, we propose a variant that incorporates ex-post fairness rather than just ex-ante fairness. 

Broadly, the fair ranking algorithms can be divided into two groups: \textit{post-processing} and \textit{in-processing}. 
Post-processing algorithms process the output of a given ranking model to incorporate group-fairness guarantees about sufficient representation of every group (especially, underprivileged demographic groups) in the top positions or top prefixes \cite{celis2018ranking,geyik2019fairness,asudeh2019designing}. 
As a result, the underlying ranking model may not be optimized in anticipation of the post-processing. 
In-processing algorithms, on the other hand, incorporate fairness controls to modify the objective in learning-to-rank \cite{singh2018fairness,singh2019policy,oosterhuis2021computationally}. 
As a consequence, previous work on post-processing algorithms in fair ranking can provide ex-post (actual) guarantees on the group-wise representation in the top ranks \cite{celis2018ranking,geyik2019fairness}, whereas in-processing algorithms can only provide ex-ante (expected) guarantees on group-wise exposure \cite{singh2018fairness} or amortized individual fairness \cite{biega2018equity}. The major drawback of the existing LTR algorithms is that none of them optimize relevance while ensuring that every output ranking satisfies group-wise representation guarantees in the top ranks. Our work aims to address this gap.

% Currently, many methods for ensuring group-fairness in ranking are post-processing methods, such as the algorithms by Celis et al. \cite{celis2018ranking}, and Singh and Joachims \cite{singh2018fairness}. These algorithms require inputs consisting of relevance scores for each item to be ranked, and they post-process the output ranking to satisfy given group-fairness constraints. The algorithm by Celis et al. \cite{celis2018ranking} is deterministic, as it outputs the group-fair ranking that maximizes utility based on the given scores subject to group-fairness constraints about group-wise representation in the top ranks or top prefixes. Their algorithm crucially requires the scores and assumes the output of the underlying ranker before post-processing to be consistent with the score-based ranking. Therefore, it cannot be used as a post-processing methods for a stochastic ranking model such as the Plackett-Luce model whose output may not always be consistent with the underlying scores.

Recently \citet{gorantla2022sampling} proposed a randomized post-processing algorithm that gives ex-post group-fairness guarantees. Their algorithm works in two steps, the first step generates a random group-fair allocation of the top-$k$ positions that satisfy given group-wise representation constraints, and the second step fills these positions consistent with the intra-group ranking within each group. Their algorithm only requires the ordinal ranking with each group as input, not the individual scores for items or across-group comparisons. Their motivation to study this was unreliable comparisons, implicit bias and incomplete information in ranking. However, their first step of sampling a group-fair allocation is closely related to our work and we apply it in the optimization of ex-post group-fair Plackett-Luce model. Therefore, it is also the most relevant post-processing baseline chosen in our experiments.

\section{Ex-Post Fairness in Ranking}
\label{sec:group-fair-plackett-luce}
\paragraph{Preliminaries.} Let $\cI$ denote the set of items (or documents). 
Let $\textsf{S}_k(\cI)$ be the set of all $k$ sized permutations of the items in $\cI$.
In the learning-to-rank setup,
for any query $q$, the goal is to output the top-$k$ ranking of relevant items. 
Let $R_{q,d}$ be an indicator random variable that takes the value of $1$ if item $d$ is relevant to $q$ and $0$ otherwise.
The probability of $d$ being relevant to $q$ is represented as $\rho_d := P(R_{q,d} = 1)$.
Let $\sigma \in \textsf{S}_k(\cI)$ represent a ranking and let $\sigma(i)$ represent the item in rank $i$.
We use $\sigma(i:i')$ for any $1\le i<i'\le k$ to represent the set of items in ranks $i$ to $i'$ included in ranking $\sigma$, that is, $\sigma(i:i') := \set{\sigma(i), \sigma(i+1), \ldots, \sigma(i')}$.
Note that $\sigma(1:k)$ represents the items in the ranking as a \textit{set}, whereas $\sigma$ itself is an ordered representation of this set of items.
In the following, we drop $q$ from the notation since, in the rest of the paper, we will be working with a fixed query $q$.

Previous works have considered stochastic ranking models since they offer equity in attention distribution across items \cite{singh2019policy}. They are preferred over deterministic rankings for diversity \cite{rankingDiversity} and robustness \cite{bruch2020stochastic}.
We also study stochastic ranking models. 
We use $\pi$ to denote a stochastic ranking model (or policy) and $\Pi$ to denote the set of all stochastic ranking policies. 
The expected relevance metric for $\pi \in \Pi$ is defined as follows,
\begin{equation}
\label{eq:relevance}
    \cR(\pi) := \sum_{\sigma \in \textsf{S}_k(\cI)}\pi[\sigma]\sum_{i \in [k]} \theta_i \rho_{\sigma(i)},
\end{equation}
where $\theta_i \in \R_{\ge 0}$ are the position discounts associated with each rank $i \in [k]$ and $\pi[\sigma]$ represents the probability of sampling $\sigma$ according to $\pi$.

\paragraph{Policy Gradients for Placket-Luce.}
We use $\pipl$ to represent the Plackett-Luce (PL) model.
This is a popular stochastic ranking model that, given a prediction model $m$ that predicts log scores $m(d)$ for each item $d$, samples a ranking from the distribution defined based on the individual scores of the items as follows,
\begin{align*}
\forall \sigma \in \textsf{S}_k(\cI),\qquad\pipl[\sigma] := \prod_{i=1}^k \frac{e^{m(\sigma(i))}}{\sum\limits_{d\in \cI \setminus \sigma(1:i-1)}e^{m(d)}}.
\end{align*}
\citet{singh2019policy} have proposed using policy gradients to train a  PL ranking model to maximize expected relevance. 
They utilize the famous log trick from the REINFORCE algorithm \cite{Williams2004SimpleSG} to compute the gradients of the expected relevance metric and use stochastic gradient descent to update the parameters of the PL model. 
\citet{oosterhuis2022learning} has developed a computationally efficient way to compute the gradients of the expected relevance metric for the PL model.
As a result, PL models can now be trained efficiently to maximize the expected relevance.

\paragraph{Group-Fair Ranking.}
Suppose that the set of items $\cI$ can be partitioned into $\ell$ groups $\cI_1, \cI_2, \ldots, \cI_\ell$ based on the group membership (based on age, race, gender, etc.).
For any integer $t$, we use $[t]$ to denote the set $\{1,2,\ldots, t\}$.
We consider the representation-based group fairness constraints in the top-$k$ rankings, where, for each group $j \in [\ell]$, we are given a lower bound $L_j$ and an upper bound $U_j$ on the number of items that can appear in the top-$k$ ranking from that group.
Let $\sfair$ represent all possible group-fair top-$k$ rankings. That is,
\begin{align*}
\sfair := \left\{\sigma \in \textsf{S}_k(\cI): L_j \le |\sigma(1:k) \cap \cI_j| \le U_j, \forall j \in [\ell] \right\}.
\end{align*}
Let $\textsf{G}_k(\ell)$ represent the set of all \textit{group assignment}s of the top-$k$ rankings for $\ell$ groups. That is, for any element $\gamma \in \textsf{G}_k(\ell)$, $\gamma(i)$ represents the group of the item in rank $i$.
\[
\textsf{G}_k(\ell) := \set{\gamma \in [\ell]^k},~~\text{or equivalently,}~~\textsf{G}_k(\ell) :=[\ell]^k.
\]
Let $g : \cI \rightarrow [\ell]$ be the group membership function for the items. 
For an item $d$, $g(d)$ represents its group membership.
We use $g(\sigma)$ to represent the vector of group memberships of the items in the ranking $\sigma$.
Note that for any $\sigma$, $g(\sigma) \in \textsf{G}_k(\ell)$.
We can then define $\gfair$ as the set of group assignments that satisfy the group fairness constraints.
\[
\gfair := \set{g(\sigma) \in [\ell]^k: \sigma \in \sfair }.
\]
Then for any $\sigma \in \sfair$, $g(\sigma)\in \gfair$.
There is a many-to-one correspondence between $\sfair$ and $\gfair$.
% $\pmb{R}$

\begin{definition}[Ex-Post Fair Policy]
    A policy $\pi \in \Pi$ is ex-post fair if each ranking $\sigma$ sampled by the policy $\pi$ satisfies representation-based fairness constraints. That is,
    \[
    \forall \sigma\sim \pi,~~g(\sigma)\not\in \gfair~\implies~\pi[\sigma] = 0.
    \]
\end{definition} 

\paragraph{Limitations of Plackett-Luce.}
There have been two significant contributions toward fair ranking with PL models. We list them and point out their limitations below.
\begin{enumerate}
    \item \textbf{In-processing.} \citet{asudeh2019designing} and \citet{oosterhuis2021computationally} have proposed policy gradients-based optimization for expected relevance and equity of expected exposure of groups of items for PL models.
    The major drawback of these methods is that fairness is measured in expectation. Therefore, the trained PL model may not satisfy ex-post fairness.
    \item \textbf{Post-processing.} \citet{singh2018fairness, celis2018ranking, gorantla2022sampling, geyik2019fairness} and many other previous works have proposed algorithms to post-process the scores or the ranking output by any LTR model (or specifically PL) to satisfy fairness. Ex-post fairness is satisfied in this case, but the trained LTR model is unaware of the post-processing that is going to be applied on the scores. Hence, it may end up learning a bad solution.
\end{enumerate}

We overcome these limitations by incorporating ex-post fairness during the training process of the PL-based LTR.
Towards this end, we propose to use a different objective function for the stochastic ranking models.

\paragraph{Proposed Optimization Objective.}
We ask for maximizing expected relevance over ex-post group-fair rankings. Then the fair expected relevance can be written as follows,
\begin{equation}
\label{eq:fair_relevance}
    \rfair(\pi) := \begin{cases}
        \sum\limits_{\sigma \in \sfair}\pi[\sigma]\sum\limits_{i \in [k]} \theta_i \rho_{\sigma(i)},&\text{if }\pi \text{ is ex-post fair}\\
        0 &\text{otherwise}.
    \end{cases}
\end{equation}
In general, the PL model may not satisfy ex-post fairness. Consider the case where the predicted scores of all the items by model $m$ are non-zero. Then every ranking in $\textsf{S}_k(\cI)$ is sampled with a non-zero probability in the PL model based on these scores.
Therefore, even if an optimal PL model that maximizes $\rfair$ is ex-post fair, the intermediate PL models during the training process may not be ex-post fair. Then $\rfair$ for intermediate PL models will be evaluated to $0$, resulting in all the gradients being set to $0$. Hence, we can not train the PL model with $\rfair$.
In fact, the only way to train a model for $\rfair$ is to make sure that the model always samples fair rankings.

\paragraph{Other approaches.} We could also optimize a different relevance metric $\rhat$ defined over $\sfair$, 
\[
\rhat(\pi) := \sum\limits_{\sigma \in \sfair}\pi[\sigma]\sum\limits_{i \in [k]} \theta_i \rho_{\sigma(i)}.
\]
For any ex-post fair policy $\pi$, $\rhat(\pi) = \rfair(\pi)$.
Moreover, if the fairness constraints are vacuous, that is, $L_j = 0$ and $U_j = k$, for all $j \in [\ell]$, then, 
$\rhat(\pi) = \rfair(\pi) = \cR(\pi)$.
Note that $\rhat$ does not strictly enforce ex-post fairness while training.
Hence PL model can be trained for optimizing $\rhat$.
One could use \textit{rejection sampling} to enforce ex-post fairness during and after training.
That is, to output a ranking from this model, we need to sample rankings from this model until we see a fair ranking.
However, in general, the probability of seeing a fair ranking may be very small.
For example, if the fairness constraints are such that $L_j = U_j$ for all but a constant number of groups in $[\ell]$, and the predicted scores of the items are such that from each group $j \in [\ell]$, $k$ items have a score of $1$ and others have score $0$, then the probability of seeing a fair ranking is $\frac{k^{c}}{k^{\ell}}$, where $c$ is a constant. This means that, in the PL model, one needs to sample $O(k^{\ell})$ many rankings in expectation before seeing a fair ranking\footnote{This follows from the fact that the expected value of a geometric random variable with parameter $p := \frac{k^{c}}{k^{\ell}}$ is $1/p$.}, which is computationally inefficient.
This also affects the training process since the estimate of the gradient only makes sense if we have enough samples that are fair rankings.

For these reasons, asking for stochastic ranking models that can be trained with $\rfair$ as an objective is well-motivated.
In the next section we describe our model that satisfies ex-post fairness, from which we can sample group-fair rankings efficiently.
As a result, we get an efficient algorithm to compute gradients of our proposed model for optimizing $\rfair$.
We can then use the stochastic gradient descent method to train our model.
\section{Group-Fair Placektt-Luce Model}
\label{sec:alg}
Let $\pfair$ represent the Group-Fair-PL model we propose.
In $\pfair$, we have a two-step process to sample ex-post group-fair rankings,
\begin{enumerate}
    \item Sample a top-$k$ group assignment $\gamma \in \gfair$.
    \item Sample a top-$k$ ranking $\sigma \in \sfair$ such that $g(\sigma) = \gamma$.
\end{enumerate}
Then the Group-Fair-PL model can be written as,
\begin{align}
    \pfair[\sigma] = \mu[g(\sigma)]\pfair[\sigma \mid g(\sigma)],
\end{align}
where $\mu[\cdot]$ is the distribution over $\{\gamma \in \gfair\}$, and $\pfair[\cdot \mid \gamma]$ is a conditional distribution over $\{\sigma \in \sfair : g(\sigma) = \gamma\}$.
It is clear that, to achieve ex-post fairness, we can only sample group-fair group assignments in Step 1.
For Step 2, we use PL model for items within the group for the ranks assigned to that group according to $\gamma$.
Therefore, in the Group-Fair-PL model, any group-fair ranking $\sigma \in \sfair$ is sampled with probability,
\begin{align}
    \pfair[\sigma] := \mu[g(\sigma)]\prod_{i=1}^k \frac{e^{m(\sigma(i))}}{\sum\limits_{d\in \underbrace{\cI_{g(\sigma(i))}}_{\text{items from group}} \setminus \sigma(1:i-1)}e^{m(d)}},
    \label{eq:group-fair-pl}
\end{align}
and any non-group-fair ranking is sampled with probability $0$.
Therefore, $\rfair$ defined in (\ref{eq:fair_relevance}) is always evaluated in the \textbf{if} case for our Group-Fair-PL model.
Let $\sigma_j$ be the (sub-) ranking of items from group $j$ in $\sigma$.
We use $\pipl_j$ to represent the group-wise PL model for group $j$.
Note that for any $j, j' \in [\ell]$, $\sigma_j$ and $\sigma_{j'}$ are sampled independently from $\pipl_j$ and $\pipl_{j'}$ respectively.
Given a group assignment $\gamma \in \gfair$, let $\psi_j(\gamma) \subseteq [k]$ be the subset of the ranks assigned to group $j$ according to $\gamma$.
Since $\cI_1, \ldots, \cI_{\ell}$ form a partition of $\cI$, $\psi_1(\gamma), \ldots, \psi_\ell(\gamma)$ form a partition of $[k]$.
Therefore, \Cref{eq:group-fair-pl} can be written as,
\begin{align}
    \pfair[\sigma] :&= \mu[g(\sigma)]\prod_{j=1}^\ell \prod_{i\in \psi_j(g(\sigma))} \frac{e^{m(\sigma(i))}}{\sum\limits_{d\in \cI_{g(\sigma(i))} \setminus \sigma(1:i-1)}e^{m(d)}}\notag\\
    &= \mu[g(\sigma)]\prod_{j=1}^\ell \pipl_j[\sigma_j].
    \label{eq:group-fair-pl-2}
\end{align}
\begin{lemma}
    $\pfair$ is a valid probability distribution over $\sfair$.
\end{lemma}
\begin{proof}
    It is clear that $\pfair[\sigma] \ge 0$ for each $\sigma \in \sfair$.
    Moreover, non-group-fair rankings are sampled from $\mu$ with probability $0$. 
    Therefore, $\pfair[\sigma] = 0$, for every $\sigma \not\in \sfair$. Further,
    \begin{align*}
         \sum_{\sigma \in \sfair} &\pfair[\sigma] \\
         &= \sum_{\sigma \in \sfair} \mu[g(\sigma)]\prod_{j \in [\ell]} \pipl_j[\sigma_j \mid g(\sigma)]\\
         &= \sum_{\gamma \in \gfair} \mu[\gamma]\sum_{\substack{\sigma \in \textsf{S}_k(\cI)\\\text{ s.t. }g(\sigma) = \gamma}} \prod_{j \in [\ell]} \pipl_j[\sigma_j\mid \gamma]\\
         &= \sum_{\gamma \in \gfair} \mu[\gamma]\paren{\sum_{\sigma_1 \in \textsf{S}_{\abs{\psi_1(\gamma)}}(\cI_1)} \paren{\sum_{\sigma_2 \in \textsf{S}_{\abs{\psi_2(\gamma)}}(\cI_2)} \cdots \paren{\sum_{\sigma_\ell \in \textsf{S}_{\abs{\psi_\ell(\gamma)}}(\cI_\ell)} \prod_{j \in [\ell]} \pipl_j[\sigma_j\mid \gamma]}}}\\
         &= \sum_{\gamma \in \gfair} \mu[\gamma]\left(\sum_{\sigma_1 \in \textsf{S}_{\abs{\psi_1(\gamma)}}(\cI_1)} \pipl_1[\sigma_1\mid \gamma] \left(\sum_{\sigma_2 \in \textsf{S}_{\abs{\psi_2(\gamma)}}(\cI_2)}\pipl_2[\sigma_2\mid \gamma]\right.\right.\cdots\\
         &\qquad\qquad\qquad\qquad\qquad\qquad\qquad\qquad\qquad\qquad\left.\left.\cdots\paren{\sum_{\sigma_{\ell-1} \in \textsf{S}_{\abs{\psi_\ell(\gamma)}}(\cI_\ell)} \pipl_\ell[\sigma_\ell\mid \gamma]}\right)\right)\\
         &= \sum_{\gamma \in \gfair} \mu[\gamma]\paren{\sum_{\sigma_1 \in \textsf{S}_{\abs{\psi_1(\gamma)}}(\cI_1)} \pipl_1[\sigma_1\mid \gamma] }\paren{\sum_{\sigma_2 \in \textsf{S}_{\abs{\psi_2(\gamma)}}(\cI_2)}\pipl_2[\sigma_2\mid \gamma]}\cdots \\&\qquad\qquad\qquad\qquad\qquad\qquad\qquad\qquad\qquad\qquad\qquad\cdots\paren{\sum_{\sigma_{\ell-1} \in \textsf{S}_{\abs{\psi_\ell(\gamma)}}(\cI_\ell)} \pipl_\ell[\sigma_\ell\mid \gamma]}\\
         &= \sum_{\gamma \in \gfair} \mu[\gamma]
         \prod_{j \in [\ell]}
         \sum_{\sigma_j \in \textsf{S}_{\abs{\psi_j(\gamma)}}(\cI_j)}  \pipl_j[\sigma_j\mid \gamma]\\
         &= \sum_{\gamma \in \gfair} \mu[\gamma]
         \prod_{j \in [\ell]}
         1\\
         &= 1.
    \end{align*}
\end{proof}

It remains to understand what should be the distribution over the group assignments $\mu[\cdot]$.
First, we need that $\mu$ is \textit{efficiently samplable}. 
It means that even if there are exponentially-many group assignments that are fair, we need to able to sample one group assignment from $\mu$ in polynomial time.
Moreover, the gradients of this policy should be \textit{efficiently computable}, meaning that the gradient of the expected relevance metric can be computed in polynomial time.

\begin{algorithm}[tb]
\caption{Group-Fair-PL}
\label{alg:algorithm}
\textbf{Input:} items: $\mathcal{I}$, relevance scores: $\rho$, ranking metric weights: $\theta$, prediction model: $m$, number of samples: $M$, groups: $\cI_1, \cI_2, \ldots, \cI_\ell$, bounds: $L, U$, ranking model: $\pfair$\\
\textbf{Output}: Gradients $\frac{\delta}{\delta m}\rfair(\pfair)$.
\begin{algorithmic}[1] %[1] enables line numbers
\STATE Sample $M$ group assignments $\gamma^{(1)}, \gamma^{(2)}, \ldots, \gamma^{(M)}$ from \cite{gorantla2022sampling} with fairness constraints $L, U$.

\FOR{each group $j = 1, 2, \ldots, \ell$}
\FOR{$t = 1,2,\ldots, M$}
\STATE Sample $\sigma^{(t)}_{j}$ from $\pipl_j$ for ranks $\psi_j(\gamma)$.
\ENDFOR
\STATE For all $d \in \cI$, estimate the gradient $\frac{\delta}{\delta m(d)}\cR_j(\pipl_j)$ with $M$ samples $\sigma_j^{(1)}, \sigma_j^{(2)}, \ldots, \sigma_j^{(M)}$ using PL-Rank-3 from \cite{oosterhuis2022learning} and  call it $\zeta_j(d)$.
\ENDFOR
\STATE \textbf{return} gradient estimated for each $d \in \cI$ as follows,
\[\frac{\delta}{\delta m(d)}\rfair(\pfair)=
\sum_{j \in [\ell]}\zeta_j(d).
\]
\end{algorithmic}
\end{algorithm}

The distribution from \cite{gorantla2022sampling} to sample the fair group assignment $\gamma$ is efficiently samplable and the gradients in this model are efficiently computable. In fact, the distribution does not depend on the predicted scores of the items. Hence, the computation of the gradients boils down to computation of the gradients for each of the group-wise PL models, which we can do efficiently owing to the PL-Rank-3 algorithm in \cite{oosterhuis2022learning}. 

The distribution $\mu$ from \cite{gorantla2022sampling} samples a fair group assignment $\gamma$ in two steps:
\begin{enumerate}
    \item First, it samples a tuple from the set $\{(x_1, x_2, \ldots, x_\ell) \in [k]^\ell : L_j \le x_j \le U_j\land x_1+x_2+\cdots+x_\ell = k\}$, uniformly at random. In this tuple, $x_j$ gives the number of items from group $j$ to be sampled in the top-$k$ ranking.
    \item Then it samples a group assignment $\gamma = (\gamma_1, \gamma_2, \ldots, \gamma_k)$ to be a uniform random permutation of the vector $(\underbrace{1,1,\ldots,1}_{x_1\text{ times}}, \underbrace{2,2,\ldots,2}_{x_2\text{ times}}, \ldots, \underbrace{\ell,\ell,\ldots,\ell}_{x_\ell\text{ times}})$.
\end{enumerate}
Below, we re-state their theorem about the time taken to sample a fair group assignment from this distribution.
\begin{theorem}[Theorem 4.1 in \cite{gorantla2022sampling}]
\label{thm:gdl}
    There is a dynamic programming-based algorithm that samples a group assignment $\gamma$ in time $O(k^2\ell)$.
\end{theorem}
Therefore, this distribution is efficiently samplable. Moreover, this distribution also satisfies additional desirable properties, which we will discuss further in \Cref{sec:experiments}. 
Our following result then shows that using the distribution given by \cite{gorantla2022sampling} for $\mu$ in our Group-Fair-PL model gives us an efficient algorithm to compute gradient of $\rfair$ with respect to the predicted scores $m$. 
\begin{theorem}
\label{thm:main}
\Cref{alg:algorithm} estimates the gradient of the relevance metric $\rfair$ in the Group-Fair-PL model in time $O\paren{Mk^2\ell + M\paren{\abs{\cI}+k\ell\log{\abs{\cI}}}}$.
\end{theorem}
\begin{proof}[Proof of \Cref{thm:main}]
Note that given a group assignment $\gamma$, the probability of sampling an item $d$ at rank $i$ depends only on the items from group $\gamma(i)$ that appear in ranks $1$ to $i-1$.
In our Group-Fair PL model, only items from group $\gamma(i)$ are sampled in rank $i$.
Let $\psi_j(\gamma)$ represent the set of ranks assigned to group $j$ according to the group assignment $\gamma$, for each $j \in [\ell]$.
Then,
\begin{align*}
    \rfair(\pfair) &= \sum_{\sigma \in \sfair}\pfair[\sigma]\paren{\sum_{i \in [k]} \theta_i \rho_{\sigma(i)}}&\\
    &=  \sum_{\sigma \in \sfair}\sum_{\gamma \in \gfair}\pfair[\gamma, \sigma]\paren{\sum_{i \in [k]} \theta_i \rho_{\sigma(i)}}&\\
    &=  \sum_{\sigma \in \sfair}\sum_{\gamma \in \gfair}\mu[\gamma]\cdot\pfair[\sigma\mid \gamma]  \paren{\sum_{i \in [k]} \theta_i \rho_{\sigma(i)}}&\\
    &= \sum_{\gamma \in \gfair}\mu[\gamma] \sum_{\sigma \in \sfair}\pfair[\sigma\mid \gamma]\paren{\sum_{i \in [k]} \theta_i\rho_{\sigma(i)}}.
\end{align*}
\Cref{eq:group-fair-pl-2} gives us,
\begin{align*}
    \rfair(\pfair) &= \sum_{\gamma \in \gfair}\mu[\gamma]\paren{\sum_{\sigma \in \sfair}\paren{\prod_{j \in [\ell]}\pipl_j[\sigma_j \mid \gamma]} \paren{\sum_{i \in [k]]} \theta_i \rho_{\sigma(i)}}}.&
\end{align*}
Now since $\psi_1(\gamma), \psi_2(\gamma), \ldots, \psi_\ell(\gamma)$ form a partition of $[k]$ we can rearrange the terms to get the following,
\begin{align*}
    &\rfair(\pfair) \\
    &= \sum_{\gamma \in \gfair}\mu[\gamma]\paren{\sum_{\sigma \in \sfair}\paren{\prod_{j \in [\ell]}\pipl_j[\sigma_j \mid \gamma]}  \paren{\sum_{j \in [\ell]}\sum_{i \in \psi_j(\gamma)} \theta_i \rho_{\sigma(i)}}}\\
    &= \sum_{\gamma \in \gfair}\mu[\gamma]\left(\sum_{\sigma_1 \in \textsf{S}_{\abs{\psi_1(\gamma)}}(\cI_1)} \sum_{\sigma_2 \in \textsf{S}_{\abs{\psi_2(\gamma)}}(\cI_2)}\cdots \right.\\
    &\left.\qquad\qquad\qquad\cdots \sum_{\sigma_\ell \in \textsf{S}_{\abs{\psi_\ell(\gamma)}}(\cI_\ell)} \paren{\prod_{j \in [\ell]} \pipl_j[\sigma_j\mid \gamma]}\paren{\sum_{j \in [\ell]}\sum_{i \in \psi_j(\gamma)} \theta_i \rho_{\sigma(i)}}\right)\\
    &= \sum_{\gamma \in \gfair}\mu[\gamma]\left(\sum_{\sigma_1 \in \textsf{S}_{\abs{\psi_1(\gamma)}}(\cI_1)} \pipl_1[\sigma_1\mid \gamma]\sum_{\sigma_2 \in \textsf{S}_{\abs{\psi_2(\gamma)}}(\cI_2)} \pipl_2[\sigma_2\mid \gamma]\cdots\right.\\
    &\left.\qquad\qquad\qquad\qquad\qquad\cdots \sum_{\sigma_\ell \in \textsf{S}_{\abs{\psi_\ell(\gamma)}}(\cI_\ell)} \pipl_\ell[\sigma_\ell\mid \gamma]\paren{\sum_{j \in [\ell]}\sum_{i \in \psi_j(\gamma)} \theta_i \rho_{\sigma(i)}}\right).
\end{align*}
Now the last term can be written as,
\begin{align*}
\sum_{\sigma_\ell \in \textsf{S}_{\abs{\psi_\ell(\gamma)}}(\cI_\ell)} &\pipl_\ell[\sigma_\ell\mid \gamma]\paren{\sum_{j \in [\ell]}\sum_{i \in \psi_j(\gamma)} \theta_i \rho_{\sigma(i)}} \\
&= \sum_{\sigma_\ell \in \textsf{S}_{\abs{\psi_\ell(\gamma)}}(\cI_\ell)} \pipl_\ell[\sigma_\ell\mid \gamma]\paren{\sum_{i \in \psi_\ell(\gamma)} \theta_i \rho_{\sigma(i)}+\sum_{j \in [\ell-1]}\sum_{i \in \psi_j(\gamma)} \theta_i \rho_{\sigma(i)}} \\
&= \sum_{\sigma_\ell \in \textsf{S}_{\abs{\psi_\ell(\gamma)}}(\cI_\ell)} \pipl_\ell[\sigma_\ell\mid \gamma]\paren{\sum_{i \in \psi_\ell(\gamma)} \theta_i \rho_{\sigma(i)}}\\
&\qquad\qquad\qquad\qquad+\sum_{j \in [\ell-1]}\sum_{i \in \psi_j(\gamma)} \theta_i \rho_{\sigma(i)} \underbrace{\sum_{\sigma_\ell \in \textsf{S}_{\abs{\psi_\ell(\gamma)}}(\cI_\ell)} \pipl_\ell[\sigma_\ell\mid \gamma]}_{=1}\\
&= \sum_{\sigma_\ell \in \textsf{S}_{\abs{\psi_\ell(\gamma)}}(\cI_\ell)} \pipl_\ell[\sigma_\ell\mid \gamma]\paren{\sum_{i \in \psi_\ell(\gamma)} \theta_i \rho_{\sigma(i)}}+\sum_{j \in [\ell-1]}\sum_{i \in \psi_j(\gamma)} \theta_i \rho_{\sigma(i)}. 
\end{align*}
Taking the summation $\sum_{\sigma_{\ell-1} \in \textsf{S}_{\abs{\psi_{\ell-1}(\gamma)}}(\cI_{\ell-1})} \pipl_{\ell-1}[\sigma_{\ell-1}\mid \gamma]$ on both sides we get,
\begin{align*}
&\sum_{\sigma_{\ell-1} \in \textsf{S}_{\abs{\psi_{\ell-1}(\gamma)}}(\cI_{\ell-1})} \pipl_{\ell-1}[\sigma_{\ell-1}\mid \gamma]\sum_{\sigma_\ell \in \textsf{S}_{\abs{\psi_\ell(\gamma)}}(\cI_\ell)} \pipl_\ell[\sigma_\ell\mid \gamma]\paren{\sum_{j \in [\ell]}\sum_{i \in \psi_j(\gamma)} \theta_i \rho_{\sigma(i)}} \\
&= \sum_{\sigma_{\ell-1} \in \textsf{S}_{\abs{\psi_{\ell-1}(\gamma)}}(\cI_{\ell-1})} \pipl_{\ell-1}[\sigma_{\ell-1}\mid \gamma]\left(\sum_{\sigma_\ell \in \textsf{S}_{\abs{\psi_\ell(\gamma)}}(\cI_\ell)} \pipl_\ell[\sigma_\ell\mid \gamma]\paren{\sum_{i \in \psi_\ell(\gamma)} \theta_i \rho_{\sigma(i)}}\right.\\
&\left.\qquad\qquad\qquad\qquad\qquad\qquad\qquad+\sum_{j \in [\ell-1]}\sum_{i \in \psi_j(\gamma)} \theta_i \rho_{\sigma(i)}\right)\\
&= \sum_{\sigma_\ell \in \textsf{S}_{\abs{\psi_\ell(\gamma)}}(\cI_\ell)} \pipl_\ell[\sigma_\ell\mid \gamma]\paren{\sum_{i \in \psi_\ell(\gamma)} \theta_i \rho_{\sigma(i)}}\underbrace{\sum_{\sigma_{\ell-1} \in \textsf{S}_{\abs{\psi_{\ell-1}(\gamma)}}(\cI_{\ell-1})} \pipl_{\ell-1}[\sigma_{\ell-1}\mid \gamma]}_{=1}\\
&\qquad\qquad\qquad\qquad\qquad\qquad\qquad+\sum_{\sigma_{\ell-1} \in \textsf{S}_{\abs{\psi_{\ell-1}(\gamma)}}(\cI_{\ell-1})} \pipl_{\ell-1}[\sigma_{\ell-1}\mid \gamma]\sum_{j \in [\ell-1]}\sum_{i \in \psi_j(\gamma)} \theta_i \rho_{\sigma(i)}\\
&= \sum_{\sigma_\ell \in \textsf{S}_{\abs{\psi_\ell(\gamma)}}(\cI_\ell)} \pipl_\ell[\sigma_\ell\mid \gamma]\paren{\sum_{i \in \psi_\ell(\gamma)} \theta_i \rho_{\sigma(i)}}\\
&\qquad\qquad\qquad\qquad\qquad\qquad\qquad+\sum_{\sigma_{\ell-1} \in \textsf{S}_{\abs{\psi_{\ell-1}(\gamma)}}(\cI_{\ell-1})} \pipl_{\ell-1}[\sigma_{\ell-1}\mid \gamma]\paren{\sum_{i \in \psi_{\ell-1}(\gamma)} \theta_i \rho_{\sigma(i)}}\\
&\qquad\qquad\qquad\qquad\qquad\qquad\qquad+\sum_{j \in [\ell-2]}\sum_{i \in \psi_j(\gamma)} \theta_i \rho_{\sigma(i)}.
\end{align*}
Repeating the above until we reach $j \in [1]$ in the last term, we get,
\begin{align}
    &\rfair(\pfair)\notag \\
    &= \sum_{\gamma \in \gfair}\mu[\gamma]\left(\sum_{\sigma_1 \in \textsf{S}_{\abs{\psi_1(\gamma)}}(\cI_1)} \pipl_1[\sigma_1\mid \gamma]\paren{\sum_{i \in \psi_1(\gamma)} \theta_i \rho_{\sigma(i)}} +\cdots \right.\notag \\
    &\left.\qquad\qquad\qquad\qquad\qquad\qquad\qquad\qquad\cdots+\sum_{\sigma_\ell \in \textsf{S}_{\abs{\psi_\ell(\gamma)}(\cI_\ell)}} \pipl_\ell[\sigma_\ell\mid \gamma]\paren{\sum_{i \in \psi_\ell(\gamma)} \theta_i \rho_{\sigma(i)}}\right)\notag \\
    &= \sum_{\gamma \in \gfair}\mu[\gamma]\left(\sum_{j \in [\ell]}\sum_{\sigma_j \in \textsf{S}_{\abs{\psi_j(\gamma)}}(\cI_j)} \pipl_j[\sigma_j\mid \gamma]\paren{\sum_{i \in \psi_j(\gamma)} \theta_i \rho_{\sigma(i)}} \right)\notag \\
    &= \sum_{\gamma \in \gfair}\mu[\gamma]\paren{\sum_{j \in [\ell]}\cR_j(\pipl_j)},\label{eq:reward_sum}
\end{align}
where $R_j(\gamma)$ is the reward obtained from the group-wise Plackett-Luce model $\pipl_j$ for the ranks assigned to group $j$ according to the group assignment $\gamma$. That is,
\[
\cR_j(\pipl_j) := \sum_{\sigma_j \in \textsf{S}_{\abs{\psi_j(\gamma)}}(\cI_j)}\pipl_j[\sigma_j \mid \gamma] \paren{\sum_{i \in \psi_j(\gamma)}\theta_i \rho_{\sigma(i)}}.
\]
Now for a fixed an item $d$, the derivative with respect to the score of $d$ will be,
\begin{align*}
    \frac{\delta}{\delta m(d)} \rfair(\pfair) &= \frac{\delta}{\delta m(d)}\sum_{\gamma \in \gfair}\mu[\gamma]\Bigg(\sum_{j \in [\ell]}\cR_j(\pipl_j)\Bigg).
\end{align*}
Since in our group-fair PL model, the group assignment $\gamma$ is sampled independently of the score $m(d)$, we have
\begin{align}
\frac{\delta}{\delta m(d)}\rfair(\pfair) &= \sum_{\gamma \in \gfair}\mu[\gamma]\Bigg(\sum_{j \in [\ell]}\frac{\delta}{\delta m(d)} \cR_j(\pipl_j)\Bigg) = \E_{\gamma \sim \mu}\sparen{\sum_{j \in [\ell]}\frac{\delta}{\delta m(d)} \cR_j(\pipl_j)}.\label{eq:nsample}
\end{align}
\paragraph{Na\"ively applying PL-Rank-3.}
PL-Rank-3 with $N$ samples can estimation the gradient $\frac{\delta}{\delta m(d)} \cR_j(\pipl_j)$, for a fixed $\gamma$, in time $O\paren{N\paren{\abs{\cI_j} + k \log\abs{\cI_j}}}$ for group $j \in [\ell]$.
Let us say we take $M$ samples to estimate the outer expectation. 
From \Cref{thm:gdl} we have that the time taken to sample one group assignment $\gamma$ is $O\paren{k^2\ell}$.
Therefore, to sample $M$ group assignments, it takes time $O\paren{Mk^2\ell}$.
Then, the total time taken to compute $\frac{\delta}{\delta m(d)}\rfair(\pfair)$ will be
\begin{align}
    O\paren{M\paren{k^2\ell + \sum_{j \in [\ell]} N\paren{\abs{\cI_j} + k \log\abs{\cI_j}}}} = O\paren{Mk^2\ell + MN\paren{\abs{\cI}+k\ell\log{\abs{\cI}}}}.\label{eq:naive_time}
\end{align}
\paragraph{Correctness for $N=1$.}
% We first fix a group $j \in [\ell]$ and a group-wise ranking $\sigma_j$ sampled from $\pipl_j$.
% Recall that $\psi_j(\gamma)$ is the \textit{set} of ranks assigned to group $j$ according to $\gamma$.
% We then use $\widetilde{\psi}_j$ to represent the \textit{ordered tuple} of ranks assigned to group $j$ according to $\gamma$. 
% For example if $\gamma = (1,1,2,1,1,3,1,2,\ldots)$, then $\widetilde{\psi}_1 = (1,2,4,5,7,\ldots)$, $\widetilde{\psi}_2 = (3,8,\ldots)$, and so on. 
% Then we define $\theta^{(j)} = \theta(\widetilde{\psi}_j)$. 
% That is, the position discount of the ranks assigned to group $j$.
% We construct a ranking problem instance specific to the group, that is, let $\rho' = (\rho_{\gamma^{-1}(j)})$.
% In the following, we use $\sigma_j^{-1}(d)$ to denote the rank assigned to document $d \in \cI_j$ within the group $j$.
% That is, $\sigma_j^{-1}(d)$ is the rank assigned to document $d$. 
% For documents not in the top $k$ rankings of $\sigma$, we use $\sigma^{-1}(d) = k+1$.
Let $\text{rank}(\sigma, d)$ represent the rank assigned to item $d$ in $\sigma$. Then from PL-Rank-3 algorithm in \cite{oosterhuis2022learning} we know that 
\begin{align}
    \frac{\delta}{\delta m(d)} \cR_j(\pipl_j) = \E_{\sigma_j\mid \gamma}\left[PR^{(j)}_{\sigma, d} + e^{m(d)}\paren{\rho_d DR^{(j)}_{\sigma,d} - RI^{(j)}_{\sigma, d}}\right],\label{eq:group_grad}
\end{align}
where
\begin{align*}
    PR^{(j)}_{\sigma, i} =\sum_{i' = [i+1,k] \cap \psi_j(\gamma)}\theta_{i'}\rho_{\sigma(i')}
    ~~~~&\text{and}~~~~ PR^{(j)}_{\sigma, d} = PR^{(j)}_{\sigma,\text{rank}(\sigma, d)},\\
    RI^{(j)}_{\sigma, i} = \sum_{i' = [i+1,k] \cap \psi_j(\gamma)}\frac{PR^{(j)}_{\sigma, i}}{\sum_{d' \in \cI_j \setminus \sigma(1:i-1)}e^{m(d')}}~~~~&\text{and}~~~~ RI^{(j)}_{\sigma, d}= RI^{(j)}_{\sigma,\text{rank}(\sigma, d)},\\
    DR^{(j)}_{\sigma, i} = \sum_{i' = [i+1,k] \cap \psi_j(\gamma)}\frac{\theta_{\sigma, i}}{\sum_{d' \in \cI_j \setminus \sigma(1:i-1)}e^{m(d')}}~~~~&\text{and}~~~~ DR^{(j)}_{\sigma, d}= DR^{(j)}_{\sigma,\text{rank}(\sigma, d)}.\\
\end{align*}
Note that for a fixed ranking $\sigma$, PL-Rank-3 computes the term inside the expectation efficiently in time $O(\abs{\cI_j} + k \log \abs{\cI_j})$.
Hence, even if the position discount values vary between different samples, or if the length of the ranking $\abs{\psi_j(\gamma)}$ changes between different samples, we can still use PL-Rank-3 algorithm to compute the term inside the expectation for each sample independently and efficiently.
Therefore, substituting \Cref{eq:group_grad} in \Cref{eq:nsample} we get,
\begin{align*}
    \frac{\delta}{\delta m(d)}\rfair(\pfair) &= \E_{\gamma \sim \mu}\sparen{\sum_{j \in [\ell]}\frac{\delta}{\delta m(d)} \cR_j(\pipl_j)} \\
    &= \E_{\gamma \sim \mu}\sparen{\sum_{j \in [\ell]}\E_{\sigma_j\mid \gamma \sim \pipl_j}\left[PR^{(j)}_{\sigma, d} + e^{m(d)}\paren{\rho_d DR^{(j)}_{\sigma,d} - RI^{(j)}_{\sigma, d}}\right]}.
\end{align*}

By linearity of expectation,
\begin{align}
    \frac{\delta}{\delta m(d)}\rfair(\pfair) &= \sum_{j \in [\ell]}\E_{\gamma \sim \mu}\sparen{\E_{\sigma_j\mid \gamma\sim \pipl_j}\left[PR^{(j)}_{\sigma, d} + e^{m(d)}\paren{\rho_d DR^{(j)}_{\sigma,d} - RI^{(j)}_{\sigma, d}}\right]}\notag\\
    &= \sum_{j \in [\ell]}\E_{\gamma,\sigma_j\sim \pfair}\left[PR^{(j)}_{\sigma, d} + e^{m(d)}\paren{\rho_d DR^{(j)}_{\sigma,d} - RI^{(j)}_{\sigma, d}}\right].\label{eq:final_grad}
\end{align}
Hence, we can estimate each term in \Cref{eq:final_grad} by taking an empirical average of $M$ samples of each group-wise rankings. 
For this, we can take $M$ samples of $\gamma$ and $1$ sample each of $\sigma_j$.
From \cite{oosterhuis2022learning} we know that for group $j$ we can compute the corresponding term in the summation in time $O(|\cI_j| + k\log|\cI_j|)$, resulting in a total time complexity of $O(Mk^2\ell + M(|\cI| + k\ell\log|\cI|))$.
Note that this is same as replacing $N = 1$ in \Cref{eq:naive_time}.
\end{proof}
Note that PL-Rank-3 takes time $O\paren{ M\paren{\abs{\cI}+k\log{\abs{\cI}}}}$ to compute the gradients.

\section{Experiments}
\label{sec:experiments}

We conduct experiments on real-world datasets to evaluate our algorithm empirically.
First, we compare Group-Fair-PL against the unconstrained PL model and observe that Group-Fair-PL is competitive in optimizing relevance. Second,  we use bias-injected data to verify that Group-Fair-PL can mitigate bias by ensuring ex-post fairness, while achieving higher true utility than the PL model. Finally, we also compare our algorithm with post-processing baselines for relevance and ex-post fairness.

\paragraph{Metrics.} We use NDCG as the ranking utility metric, with position discounts $\theta_i = \frac{1}{\log_2(i+1)}$ for all $i \in [k]$ (first row in all the figures).
% To verify ex-post fairness, we look at the \textit{fraction of rankings} curves in \textit{each} of the top $k$ ranks
The second row in the figures shows the \textit{fraction of rankings} sampled from the stochastic ranking models, where an item from the minority group is placed at rank $i$ for each rank $i \in [k]$. The minority group is as mentioned in \Cref{tab:datasets}.
The lower and upper bound lines in the figures show $(p\pm \delta)k$, where $p$ is the proportion of the minority group in the dataset and $\delta$ is a small number (see \Cref{tab:datasets}).

\paragraph{Baselines.}
Apart from PL-Rank-3, we consider \textbf{PL-Rank-3 + GDL22} and  \textbf{PL-Rank-3 + GAK19} as baselines to compare our fair in-processing algorithm Group-Fair-PL with post-processing baselines by \cite{gorantla2022sampling} and \cite{geyik2019fairness}, respectively.
GAK19 is a fairness post-processing method that also simultaneously optimizes for relevance, unlike GDL23, which only optimizes for relevance within the groups but completely ignores the inter-group comparisons.
We also compare results with \textbf{PL-Rank-3 (true)} which is the PL model trained with PL-Rank-3 on the unbiased (or true) relevance scores.
For more details regarding the parameter choices, see \Cref{tab:datasets}.
% $278547$ &
\begin{table*}[t]
    \centering
    \caption{Parameters and results of the experiments on various datasets.}
    \label{tab:datasets}
    \setcellgapes{2.5pt}\makegapedcells
    \resizebox{\textwidth}{!}{
    \begin{tabular}{|l|c|c|c|c|c|c|c|c|c|c|c|c|}
        \hline
        \multicolumn{7}{|c|}{\textbf{Dataset} }&\multicolumn{6}{c|}{\textbf{Experiment}}\\
        \hline
        \textbf{Name} &  \textbf{\#queries} &\makecell[c]{\textbf{max} \\\textbf{\#items}\\\textbf{per}\\ 
        \textbf{query}} &\makecell{\textbf{Relevance}\\ \textbf{Labels}}& \makecell{\textbf{Sensitive}\\ \textbf{feature}} & \textbf{Groups} & \textbf{Minority} & \makecell[c]{$M$\\(Alg. 1)} &$k$& $\delta$ & \makecell[c]{\textbf{Avg. running} \\ \textbf{time (sec.)}\\(\textit{Group-Fair-PL})} & \makecell[c]{\textbf{Avg. running} \\ \textbf{time (sec.)}\\(\textit{PL-Rank-3})} & \textbf{Reference}\\
        \hline
        MovieLens &  $2290$ & $588$ &\makecell[c]{$1,2,3,$\\$4,5$} & Genre &\makecell[c]{Action($33\%$), Crime($12\%$),\\Romance($30\%$), \\Musical($9\%$), Sci-Fi($16\%$)}& Crime &10&10&0.02&4285&118&\Cref{fig:movielens}\\
        \hline
        German Credit &  $500$ &$25$ & $0, 1$ & Gender & Male($74\%$), Female($26\%$) & Female &50&20&0.05&3008&59&\Cref{fig:german}\\
        \hline
        HMDA (AK) & $75$ & $25$ & $0, 1$ & Gender & Male($71\%$), Female($29\%)$&Female &100&25&0.06&1528&50&\Cref{fig:hmda}\\
        \hline
        HMDA (CT)& $731$ & $100$ & $0, 1$ & Gender & Male($67\%$), Female($33\%$)&Female &10&25&0.06&7850&673&\Cref{fig:hmda_CT}\\
        \hline
    \end{tabular}}
    
\end{table*}
\begin{figure*}[h]
    \centering
    \includegraphics[width=0.9\textwidth]{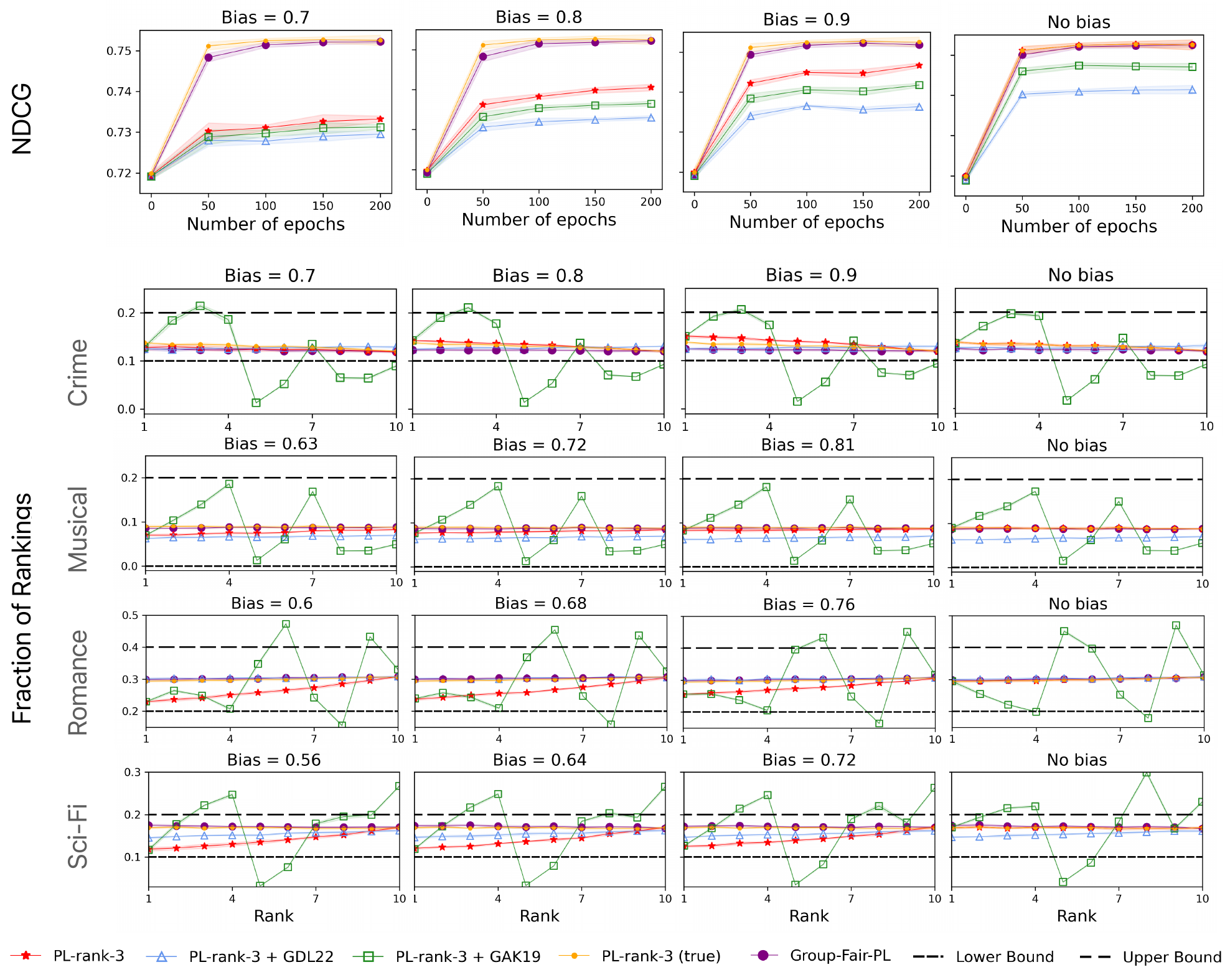}

    \caption{Results on the MovieLens dataset.}
    \label{fig:movielens}
\end{figure*}

\paragraph{Hyperparameters.}
We use a two-layered neural network of $32$ hidden units each to predict relevance scores. We use stochastic gradient descent with a learning rate of $0.001$ and batch size $512$ to optimize our relevance metric.
We report aggregate results for $10$ runs of each algorithm. We selected other hyperparameters after searching for $\delta$ in the range $0.01$ to $0.1$, $M$ in the range $10$ to $100$, and $k$ in the range $10$ to $30$. We chose the final values to be the smallest in the range where implicit bias had a significant impact on the output.

\paragraph{Implementation.}  The unconstrained PL model was trained using PL-Rank-3 algorithm from \cite{oosterhuis2022learning}.
All the experiments were run on an Intel(R) Xeon(R) Silver 4110 CPU (8 cores, 2.1 GHz clock speed, and 128GB DRAM). For reproducibility, our data and implementation of Group-Fair-PL will be uploaded at \url{github.com/sruthigorantla/Group-Fair-PL}. 
\begin{figure*}
    \centering
    \includegraphics[width=0.9\textwidth]{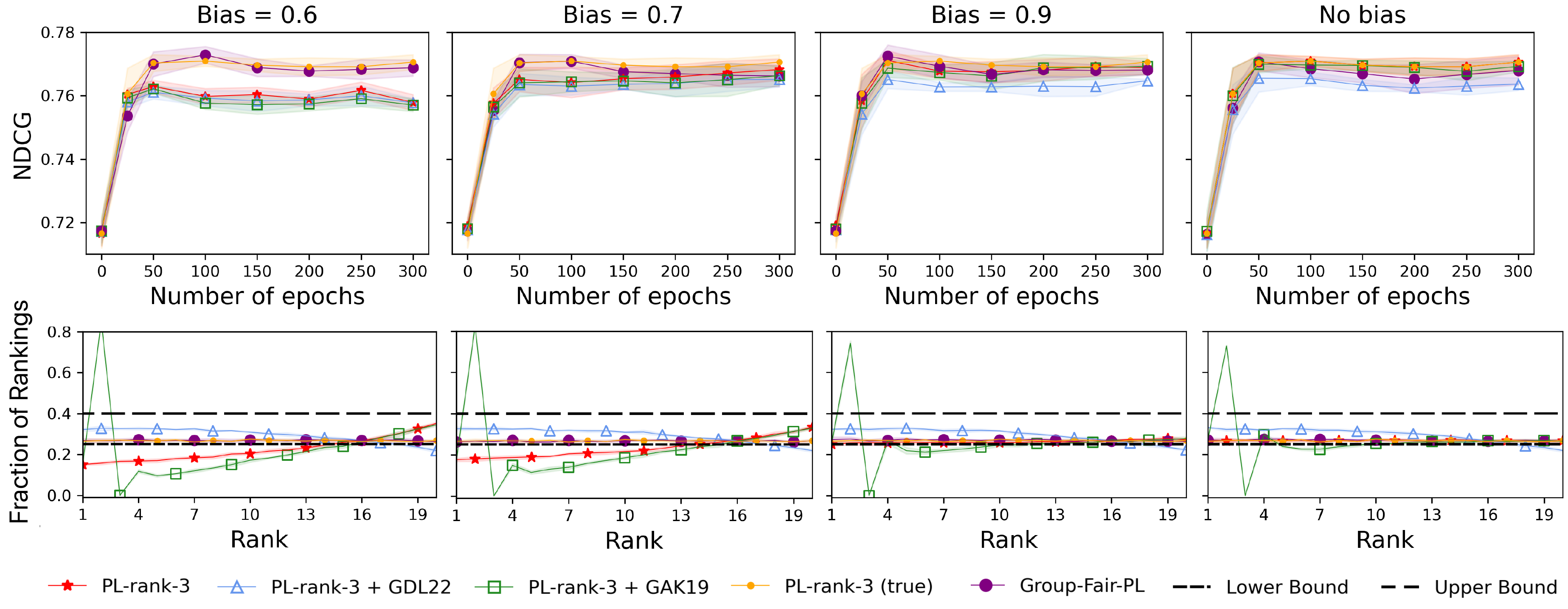}
    
    \caption{Results on the German Credit dataset.}
    \label{fig:german}
\end{figure*}

\begin{figure*}
    \centering
    \includegraphics[width=0.9\textwidth]{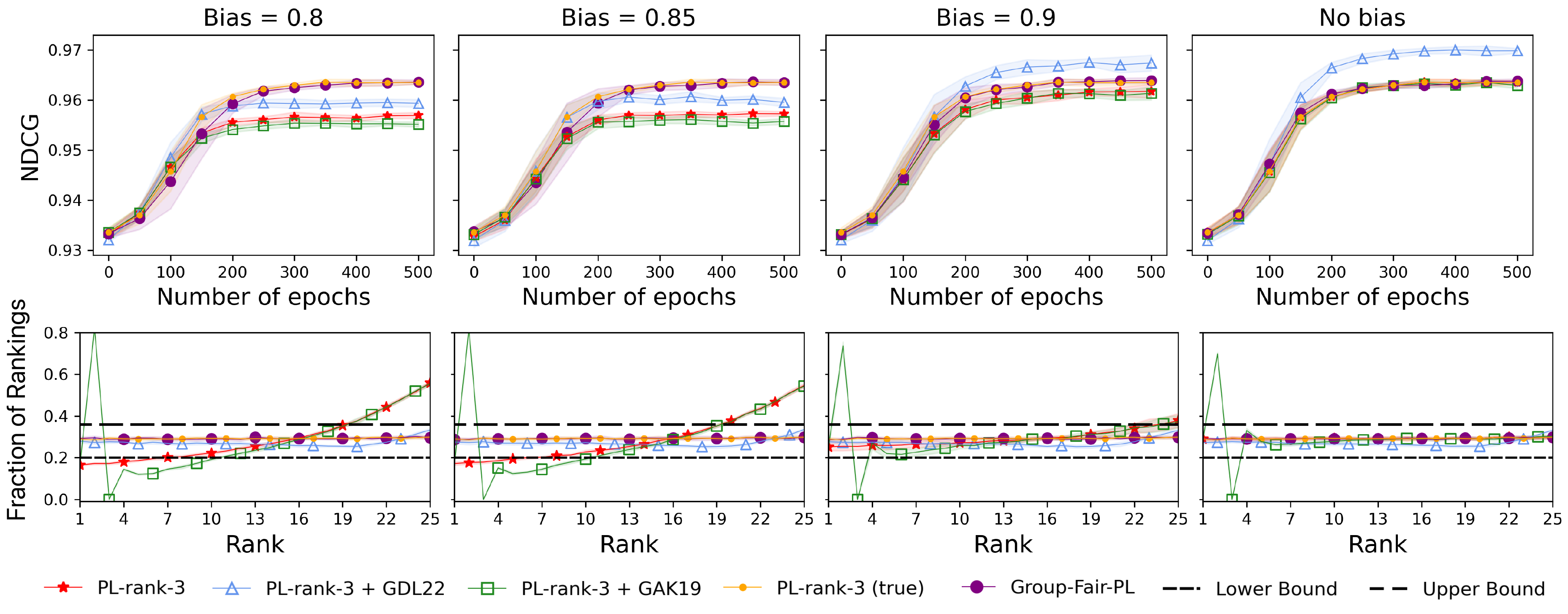}
    
    \caption{Results on the HMDA (AK) dataset.}
    \label{fig:hmda}
\end{figure*}

\begin{figure*}
    \centering
    \includegraphics[width=0.9\textwidth]{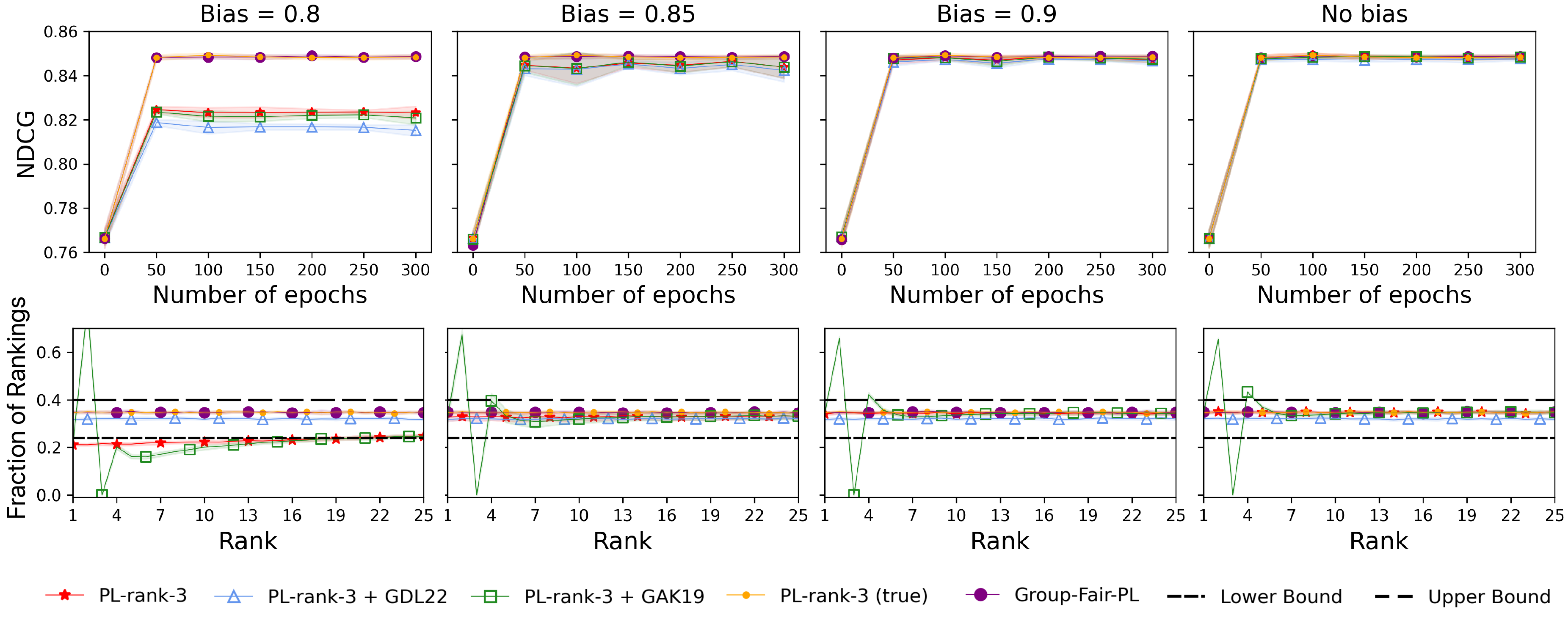}
    
    \caption{Results on the HMDA (CT) dataset.}
    \label{fig:hmda_CT}
\end{figure*}
\paragraph{Datasets.} We perform experiments on datasets, for which several past works have raised fairness concerns and demonstrated the performance of their fair ranking algorithms \cite{singh2019policy, yadav2021policy, oosterhuis2022learning,cooper2023variance}.
The \textbf{German Credit} dataset encodes users' creditworthiness as a $0/1$ label \cite{misc_statlog_(german_credit_data)_144}. To put this data into query-document pairs, we followed preprocessing similar to \cite{singh2019policy}. 
% We created $500$ queries, each of which consist of $25$ randomly selected users.
The \textbf{MovieLens} dataset consists of user ratings of movies from the movielens.org website \cite{harper2016movielens}. 
We first performed a singular value decomposition to generate $50$-dimensional features. We then chose the largest $5$ genres (see \Cref{tab:datasets}) and kept users that rated at least $50$ movies. 
The \textbf{HMDA} dataset consists of data regarding home mortgage loans in the US \cite{ffiec2022housingdata}. 
We used the preprocessed dataset released by \cite{cooper2023variance}. 
% We filtered out rows with incomplete data, and set the relevance label to be the 'action\_taken' attribute. While the original HMDA data had 8 possible values for this attribute, we categorized these into loans that were accepted and loans that were rejected as described by \cite{cooper2023variance}, assigning a label of 1 or 0 respectively. 
The HMDA dataset is available for every year since $2007$, for all $50$ US states. 
We used the data for Alaska (\textbf{AK}) from $2017$ and created a train and test split. For a more rigorous testing of our algorithms, we also used Connecticut's (\textbf{CT}) data, using years $2013-2016$ as training data and year $2017$ as test data. We did a PCA pre-processing to reduce feature dimension to $50$ \cite{DH2004} and created query-document pairs similar to the German Credit pre-processing in \cite{singh2019policy}.
The details of the datasets are in \Cref{tab:datasets}.

\paragraph{Datasets with implicit bias.}
For each dataset, we inject multiplicative implicit bias in the relevance scores of the items from the minority group as a stress test for ranking algorithms.
In the HMDA dataset, we multiply the relevance scores of the \textit{female} candidates by $\beta$, where $\beta$ is varied between $0$ and $1$ across the columns of \Cref{fig:hmda}.
For datasets with more than two groups, such as \textit{MovieLens}, we use different values of bias for different groups. We report the bias values for all the groups other than \textit{Action} group.
This model of bias is inspired by \cite{celis2020interventions}, a practical model that gives useful insights about the correct optimization objective to consider.

\subsection{Key Observations}
\paragraph{Group-Fair-PL gets the best of both fairness and NDCG.} 

In the presence of implicit bias, Group-Fair-PL outperforms PL-Rank-3 in the NDCG computed on the true scores and achieves almost same NDCG as PL-Rank-3 (true).
Compared to just post-processing for ex-post fairness (PL-Rank-3 + GDL22 and PL-Rank-3 + GAK19), our algorithm almost always achieves better NDCG, 
This suggests that by explicitly enforcing ex-post fairness during training, we are able to overcome implicit bias via eliminating unreliable comparisons of items from different groups -- main motivation of \cite{gorantla2022sampling}.
Even when there is no bias, our Group-Fair-PL still outputs ex-post-fair rankings while not compromising on the NDCG.

% \paragraph{Group-Fair-PL helps Mitigate Implicit Bias.}
\paragraph{Group-Fair-PL preserves the fairness properties of \cite{gorantla2022sampling}.}
Theorem 3.5 of \citet{gorantla2022sampling} says that with their group assignment sampling, the probability of each group being ranked at \textit{any} of the top $k$ ranks is between the lower and the upper bound (see \Cref{fig:hmda} row 2), even though the constraints are only for the representation of the group in the whole of top $k$ ranks.
This guarantee is achieved by GDL22 post-processing anyway. But even without further post-processing, Group-Fair-PL preserves this property.
In contrast, PL-Rank-3 and PL-Rank-3 + GAK19 push the protected group items to the bottom of the top $k$ ranking in the presence of bias (see \Cref{fig:hmda} row 2), even when the true representation of the minority group is uniform across the ranks (see PL-Rank-3 (true)).

\paragraph{Running time.} The running time of the algorithms is as shown in \Cref{tab:datasets}. We note that Group-Fair-PL spent most of the time sampling the group assignments. Finding faster algorithms for sampling group assignments is an interesting open problem.

\section{Conclusion}
We propose a novel group-fair Plackett-Luce model for stochastic ranking and show how one can optimize it efficiently to guarantee high relevance along with guaranteed ex-post group-fairness instead of ex-ante fairness known in previous literature on fair learning-to-rank. We experimentally validate the fairness and relevance guarantees of our ranking models on real-world datasets. Extending our results to more stochastic ranking models in random utility theory is an important direction for future work.

\section*{Acknowledgements}
SG was supported by a Google PhD Fellowship. AL is grateful to Microsoft Research for supporting this collaboration.
\bibliographystyle{unsrtnat}
\bibliography{references}

\end{document}